\documentclass[journal]{IEEEtran}
\usepackage{graphicx}
\usepackage[utf8]{inputenc}
\usepackage{graphics} 
\usepackage{epsfig} 
\usepackage{mathptmx} 
\usepackage{amsmath} 
\usepackage{amssymb}  
\usepackage{cite}
\usepackage{physics}
\usepackage{multicol}
\usepackage{mathtools, cuted}
\usepackage[cmintegrals]{newtxmath}
\usepackage{epstopdf}
\usepackage{graphics} 
\usepackage{epsfig} 

\usepackage{mathptmx} 
\usepackage{amsmath} 
\usepackage{amssymb}  
\usepackage{cite}
\usepackage{multicol}
\usepackage{mathtools, cuted}
\usepackage[cmintegrals]{newtxmath}
\usepackage{mathtools}

\usepackage{mathptmx}
\usepackage{hyperref}
\usepackage{helvet}
\usepackage{courier}
\usepackage{subfigure}
\usepackage{type1cm}
\usepackage{titlesec}
\usepackage{epsfig} 
\usepackage{mathptmx} 
\usepackage{amsmath} 
\usepackage{amssymb}  
\usepackage{cite}
\usepackage{multicol}
\usepackage{mathtools, cuted}
\usepackage[cmintegrals]{newtxmath}
\usepackage{makeidx}         
\usepackage{algorithm}
\usepackage{algpseudocode}

\usepackage{multicol}        
\usepackage[bottom]{footmisc}
\usepackage{caption}
\usepackage{nomencl}
\usepackage[usenames, dvipsnames]{color}
\usepackage{url}
\usepackage{siunitx}
\usepackage[utf8]{inputenc}
\usepackage{url}

\usepackage{multicol}        
\usepackage[bottom]{footmisc}
\usepackage{caption}
\usepackage{titlesec}
\usepackage{nomencl}
\usepackage[usenames, dvipsnames]{color}
\usepackage{multirow}
\usepackage{adjustbox}
\usepackage{amsthm}
 
\theoremstyle{definition}

\newtheorem{theorem}{Theorem}

\theoremstyle{remark}

\newtheorem{assumption}{Assumption}
\ifCLASSINFOpdf
\else
\fi
\begin{document}
%
\title{Affine Transformable Unmanned Ground Vehicle}
%
%
%

\author{Aron Mathias, Mohammad Ghufran, Jack Hughes, 
        and~Hossein Rastgoftar
\thanks{Authors are with the Aerospace and Mechanical Engineering, University of Arizona, Tucson, Arizona 85721 Emails: \{aronmathias, ghufran1942, jath03, hrastgoftar\}@arizona.edu}
}

\maketitle

\begin{abstract}
This paper develops the proof of concept for a novel affine transformable unmanned ground vehicle (ATUGV) with the capability of safe and aggressive deformation while carrying multiple payloads. The ATUGV is a multi-body system with mobile robots that can be used to power the ATUGV morphable motion, powered cells to enclose the mobile robots, unpowered cells to contain payloads, and a deformable structure to integrate cells through bars and joints. The paper's main objective is that all powered and unpowered cells' motion can safely track a desired affine transformation, where an affine transformation can be decomposed into translation, rigid body rotation, and deformation. To this end, the paper first uses a deep neural network to structure cell interconnection in such a way that every cell can freely move over the deformation plane, and the entire structure can reconfigurably deform to track a desired affine transformation. Then, the mobile robots, contained by the powered cells and stepper motors, regulating the connections of the powered and unpowered cells, design the proper controls so that all cells safely track the desired affine transformation. The functionality of the proposed ATUGV is validated through hardware experimentation and simulation.

\end{abstract}


%
\IEEEpeerreviewmaketitle

\section{Introduction}
Unmanned Ground Vehicles (UGVs) have become increasingly important across various sectors due to their ability to perform risky tasks without human interventions, while they can potentially offer the highest level of autonomy as the result of advancements in artificial intelligence. 
This paper presents a proof of concept for a novel morphable UGV that can deformably pass through constrained environmentdsf while it can carry multiple payloads.


%
\begin{figure}[ht]
\centering
\includegraphics[width=0.48 \textwidth]{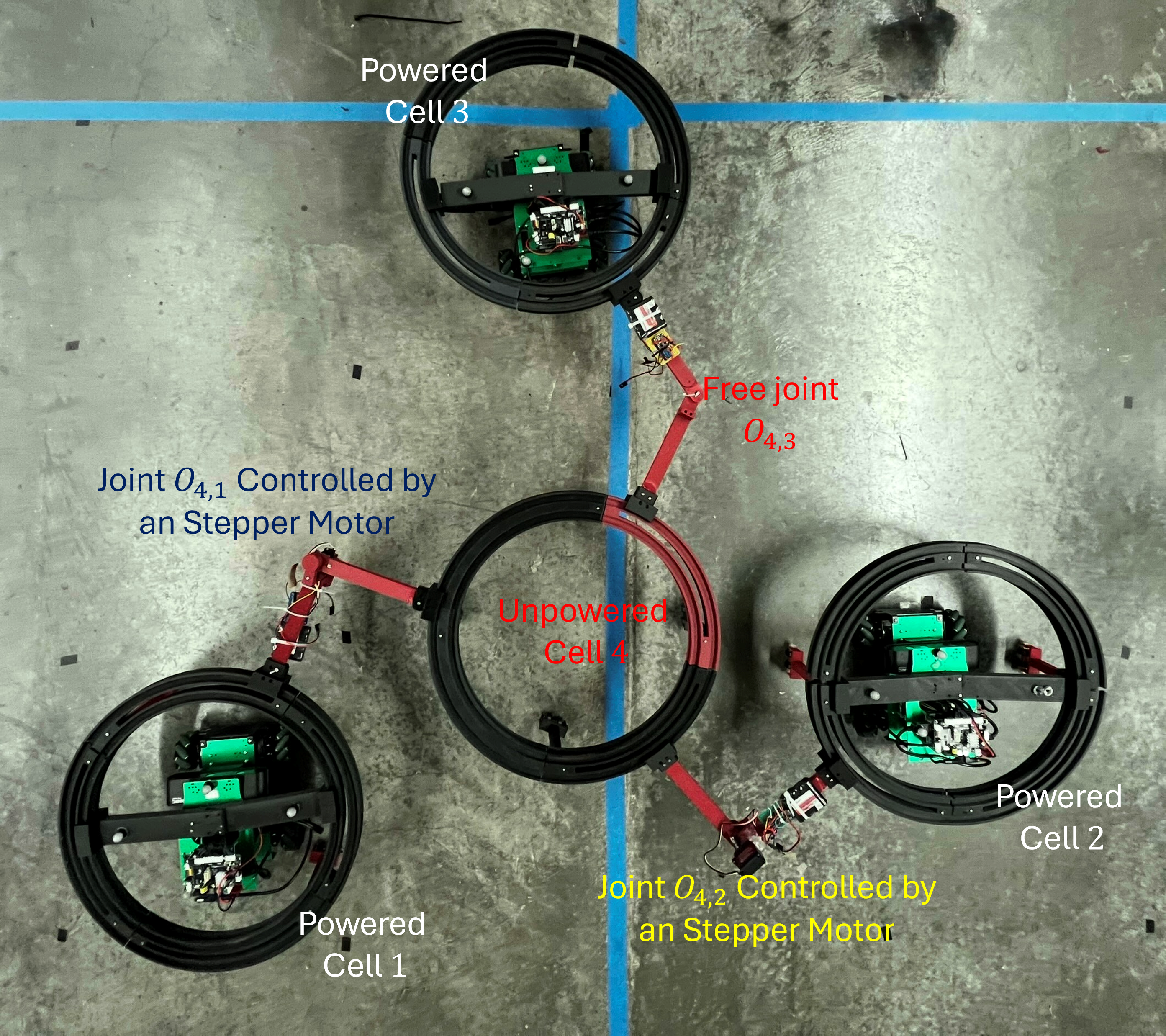}
\caption{An example ATUGV with three powered cells and one unpowered cell.}
\label{DGUConfiguration} 
\end{figure}
\vspace{-0.3cm}
\subsection{Related Work}
Mobile robots are flexible and stable systems that have found a variety of applications in manufacturing and logistics \cite{cognominal2021evolving, fragapane2022increasing}, agriculture \cite{yepez2023mobile, sun2021enabling}, military and defense-related operations \cite{gnanaprakasam2023novel, ali2021artificial, appelqvist2010mechatronics}, transportation \cite{raj2022comprehensive, zhang2022cooperative}, construction \cite{li2022map, asadi2021automated}, space exploration \cite{gul2021multi, seeni2010robot}, and many others.
They have been widely used for off-road applications as they can navigate rough surfaces and terrains \cite{rastgoftar2018data}. The DARPA Grand Challenge and PerceptOR program have advanced autonomous driving through improved perception and decision-making. In this context, Ref. \cite{kelly2006toward} introduced a three-tier architecture for off-road navigation. In \cite{reina2016lidar} and \cite{gonzalez2017thermal},  LIDAR, stereo video, and thermal cameras are used for real-time traversability assessment. 

Unlike the available mobile robots, this paper provides a design and model for an affine transformable unmanned ground vehicle (ATUGV). This novel concept adds the capability of reconfiguration and carrying multiple distributed payloads to existing mobile robots. As a result, ATUGV safely deforms to pass constrained environments in addition to its capability to navigate rough surfaces. Affine transformations of individual agents have been previously studied in \cite{chen2019ordered, rastgoftar2021safe, rastgoftar2021scalable} but affine transformations of an integrated mobile robot have never been studied in the past. It has been shown that an affine transformation of a multi-agent system can be defined as a decentralized leader-follower problem with communication weights that are assigned based on the reference configuration of the agent team. In this context, the elements of the Jacobian and rigid-body translation of an affine transformation in $\mathbb{R}^n$ can be defined based on $n+1$ leaders forming an $n$-D complex in $\mathbb{R}^n$ and precisely acquired by followers through local interagent communication \cite{rastgoftar2022spatio}. 

In the literature, researchers have developed the concept of robots with deformation and morphing capabilities, primarily designed for aerial applications.   Ref. \cite{kumar2020flight} develops model and control for a quadcopter with the capability of morphing during the flight. The concept of a single multi-copter with the capability of transitioning between quadcopter, hexacopter, octocopter, and decacopter models was developed in Ref. \cite{niemiec2018control}. Authors of Ref. \cite{zhao2017deformable} modeled the kinematics and dynamics of a single quadcopter with deformable rotor arms. Researchers have also developed fault-tolerant control models for multi-copters which enable them to recover a safe flight in the presence of actuation failure \cite{quan2021controllable, watanabe2022fault}. The corresponding author has extensively investigated the problem of continuum deformation of ``free'' quadcopters over the past few years. In Refs. \cite{rastgoftar2021safe,  rastgoftar2020scalable}, quadcopters are considered as particles of a deformable body and decentralized leader-follower approaches have been developed to ensure safety of their collective and aggressive deformation in obstacle-laden environments.

\vspace{-0.3cm}
\subsection{Contributions}
This project design, implement, and experiment a new affine transformable unmanned ground vehicle (ATUGV) with the capability of aggressive deformation in a constrained space enabling the distributed ground payload transport and  surveillance capabilities (Fig.  \ref{DGUConfiguration}). ATUGV will be a planar UGV-structure system comprising powered and unpowered cells, mobile robots for power generation and controlling the powered cells, stepper motors for controlling the unpowered cells, links and joints for building an integrated and morphable vehicle. We note that the unpowered cells can carry payloads while the powered cells enclose mobile robots. All cells are connected by a deformable structure that allows aggressive continuum deformation of the ATUGV. Therefore, the inter-agent distances between cells can safely change and be controlled.

Compared to the existing literature, this paper has the following contributions:
\begin{itemize}
\item We propose a novel approach for structuring the ATUGV cells using a deep neural network. This structure enables an ATUGV to aggressively deform while ensuring safety.
\item We decompose the ATUGV motion into rigid-body translation, shear deformation, and rotation. By this decomposition, we can characterize the deformable motion of ATUGV and specify a lower limit on the principal strains of the ATUGV affine transformation so that  it is safely planned.
\item  We experimentally validate the proposed ATUGV motion through hardware implementation using ROS-based mobile robots and servo-controlled joints, demonstrating accurate and safe tracking of affine transformations.
\end{itemize}

\vspace{-0.3cm}
\subsection{Outline}
The rest of the paper is organized as follows: The components of the ATUGV are explained in Section \ref{ATUGV Components}. The affine transformation model is developed in Section \ref{atrans} to define the desired trajectories of  powered and unpowered cells. Section \ref{atrans} also explains how an affine transformation can be tracked by the powered and unpowered cells in a safe manner. Affine transformation of the ATUGV is demonstrated by hardware experimentation and simulation and the results are presented  in Section \ref{results}. Concluding remarks and plan for the future work are discussed in Section \ref{Conclusion}.
\vspace{-0.3cm}
\section{ATUGV Structure and Components}\label{ATUGV Components}
The ATUGV is an integrated system consisting of cells and connection mechanisms that are all distributed in a two-dimensional plane. Cells are either ``powered''or "unpowered," where powered cells enclose unmanned vehicles and unpowered cells enclose payloads. We use set $\mathcal{V}=\left\{1,\cdots,N\right\}$ to define all ``powered'' and ``unpowered'' cells. Note that a payload enclosed by an ``unpowered'' cell $i\in \mathcal{V}$, or an unmanned vehicle enclosed by a powered cell $i\in \mathcal{V}$, is identified by the same index number $i\in \mathcal{V}$ in this paper. 

\begin{assumption}
    Every cell $i\in \mathcal{V}$ has the same radius $r$.
\end{assumption}

   
Cell interconnections are defined by graph $\mathcal{G}\left(\mathcal{V},\mathcal{E}\right)$, where edge set $\mathcal{E}\subset \mathcal{V}\times \mathcal{V}$ defines the cell interconnections. Given graph $\mathcal{G}\left(\mathcal{V},\mathcal{E}\right)$, set
\begin{equation}
    \mathcal{N}_i=\left\{j\in \mathcal{V}:\left(j,i\right)\in \mathcal{E}\right\}
\end{equation}
defines three neighbors of cell $i\in \mathcal{V}$. We  decompose $\mathcal{V}$ into  $N_L$ distinct groups where the group identification numbers are defined by set $\mathcal{L}=\left\{0,\cdots,N_L \right\}$. Therefore, $\mathcal{V}$ can be expressed as
\begin{equation}
    \mathcal{V}=\bigcup_{l\in \mathcal{L}}\mathcal{V}_l
\end{equation}
where $\mathcal{V}_l$ through $\mathcal{V}_{N_L}$ are disjoint subsets of $\mathcal{V}$. Particularly, set $\mathcal{V}_0=\left\{1,2,3\right\}$ defines three boundary cells at vertices of the tringle containing all other cells. The remining cells identified by $\mathcal{V}\setminus \mathcal{V}_0$ are interior cells. 
\begin{figure}[h]
\centering
\subfigure[]{\includegraphics[width=0.44\linewidth]{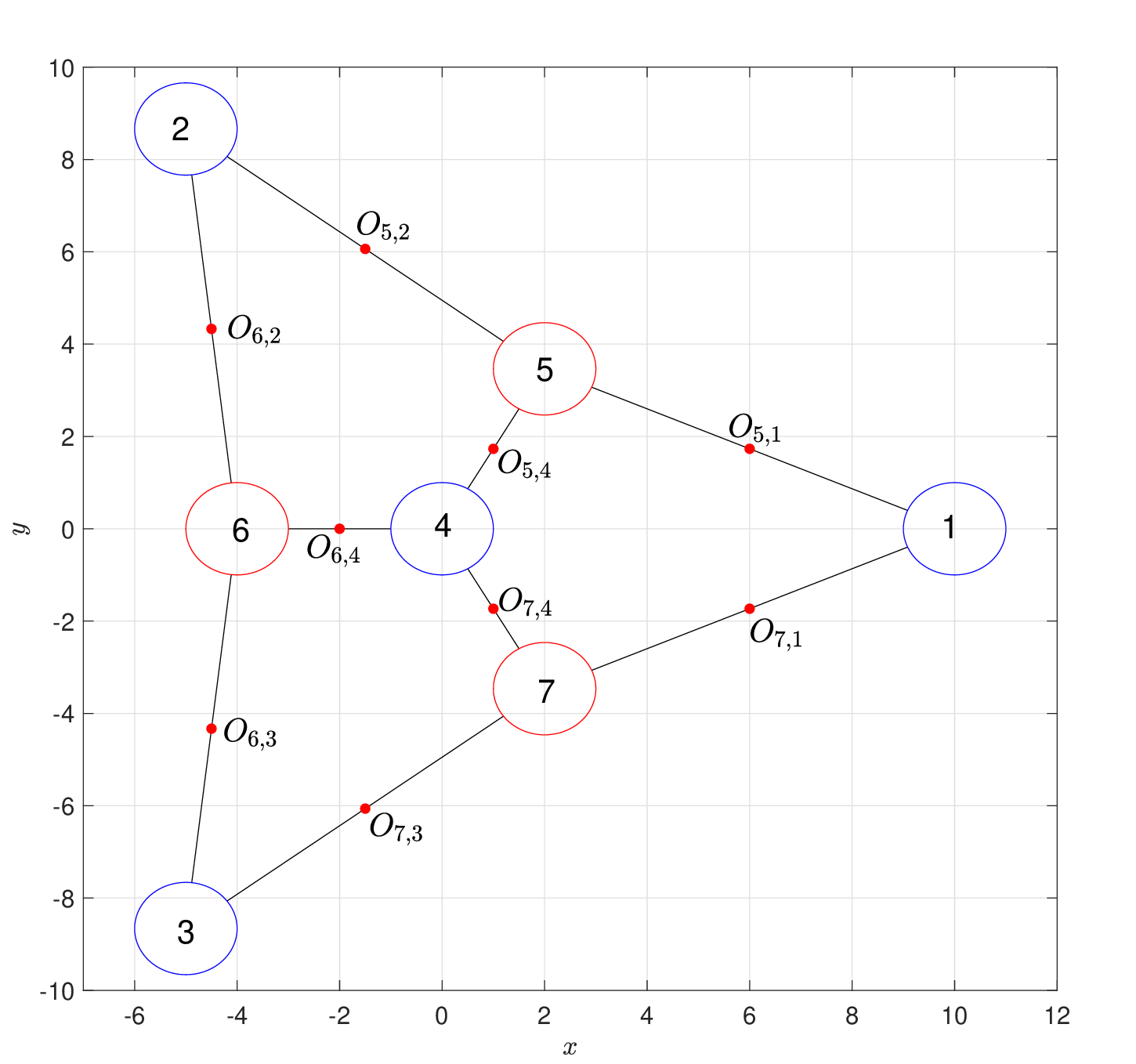}}
\subfigure[]{\includegraphics[width=0.54\linewidth]{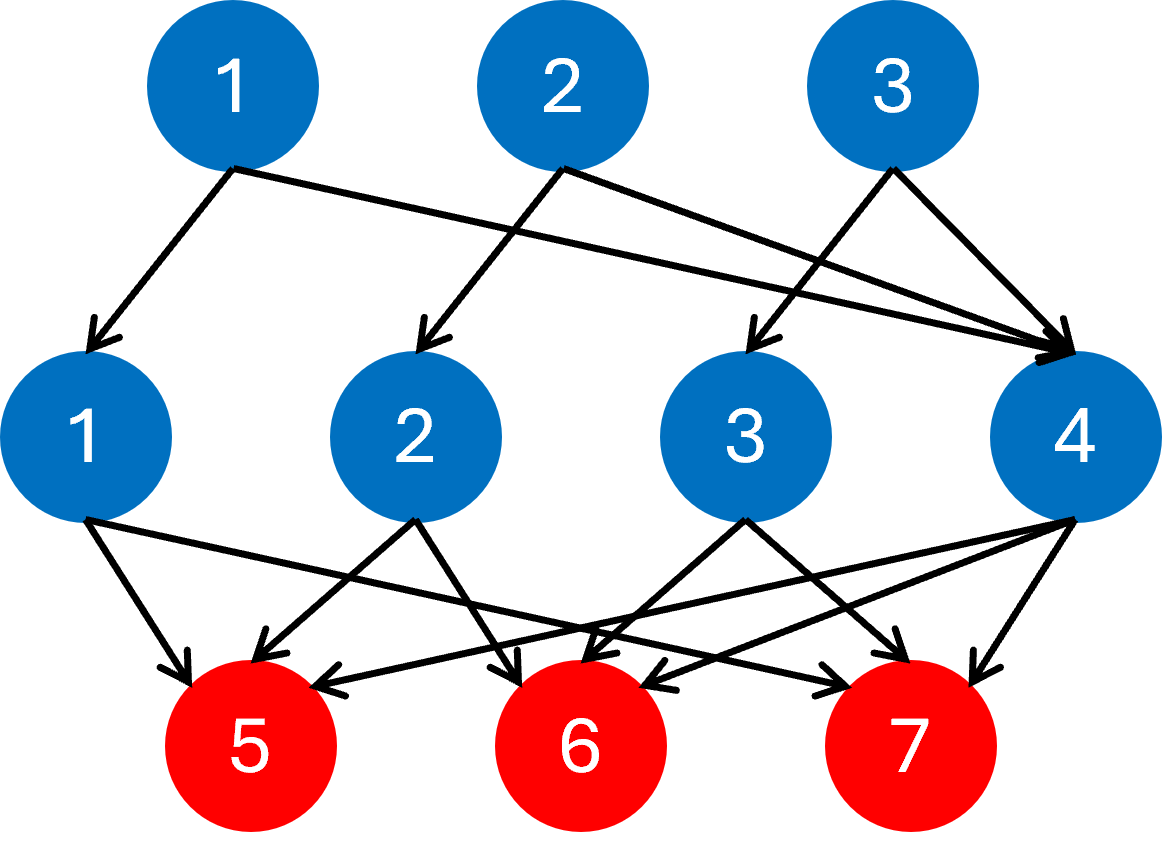}}
\vspace{-0.4cm}
\caption{(a) A seven-cell ATUGV with cells identified by $\mathcal{V}=\left\{1,\cdots,7\right\}$.  (a) The cell interconnections defined by a neural network.
}
\label{PUConnections}
\end{figure}

Without loss of generality, this paper identifies unpowered cells by $\mathcal{V}_{N_L}$. Therefore, $\mathcal{V}\setminus \mathcal{V}_{N_L}$ identifies the remaining powered cells. For better clarification, Fig. \ref{PUConnections}(a) illustrates the reference configuration of a seven-cell ATUGV with the cells identified by $\mathcal{V}=\left\{1,\cdots,7\right\}$. The cells are divided into three groups; therefore, $\mathcal{L}=\left\{0,1,2\right\}$ ($N_L=2$) $\mathcal{V}=\mathcal{V}_0\cup \mathcal{V}_2\cup \mathcal{V}_3$, where $\mathcal{V}_0=\left\{1,2,3\right\}$ identifies the boundary cells, $\mathcal{V}_1=\left\{4\right\}$, and $\mathcal{V}=2=\left\{5,6,7\right\}$ identifies the unpowered cells. Note that $\mathcal{V}\setminus \mathcal{V}_2=\left\{1,\cdots,4\right\}$ defines the powered cells.


The cell interconnections, defined by graph $\mathcal{G}\left(\mathcal{V},\mathcal{E}\right)$, are structured such that they can be represented by a deep neural network of $N_L+1$ layers where neurons of layer $l\in \mathcal{L}$ are defined by $\mathcal{W}_l\in\mathcal{V}$ as follows:
\begin{equation}
    \mathcal{W}_l=
    \begin{cases}
         \mathcal{V}_l & l\in \left\{0,M\right\}\\
         \mathcal{W}_{l-1}\cup \mathcal{V}_l&\mathrm{otherwise}
    \end{cases}
\end{equation}
Then, set 
\begin{equation}
    \mathcal{I}_{i,l}=
    \begin{cases}
        \mathcal{N}_i&i\in \mathcal{V}_l\\
        \left\{i\right\}&i\in \mathcal{W}_l\setminus \mathcal{V}_l
    \end{cases}
\end{equation}
defines the neurons of layer $l-1\in \mathcal{L}$ connected to $i\in \mathcal{W}_l$.

For the seven-cell shown in Fig. \ref{PUConnections}, the cell interconnections are specified by the neural network shown in Fig. \ref{PUConnections}(b), where $\mathcal{N}_4=\left\{1,2,3\right\}$, $\mathcal{N}_5=\left\{1,2,4\right\}$, $\mathcal{N}_6=\left\{2,3,4\right\}$, and $\mathcal{N}_7=\left\{1,3,4\right\}$.

\begin{figure}[ht]
\centering
\includegraphics[width=0.48 \textwidth]{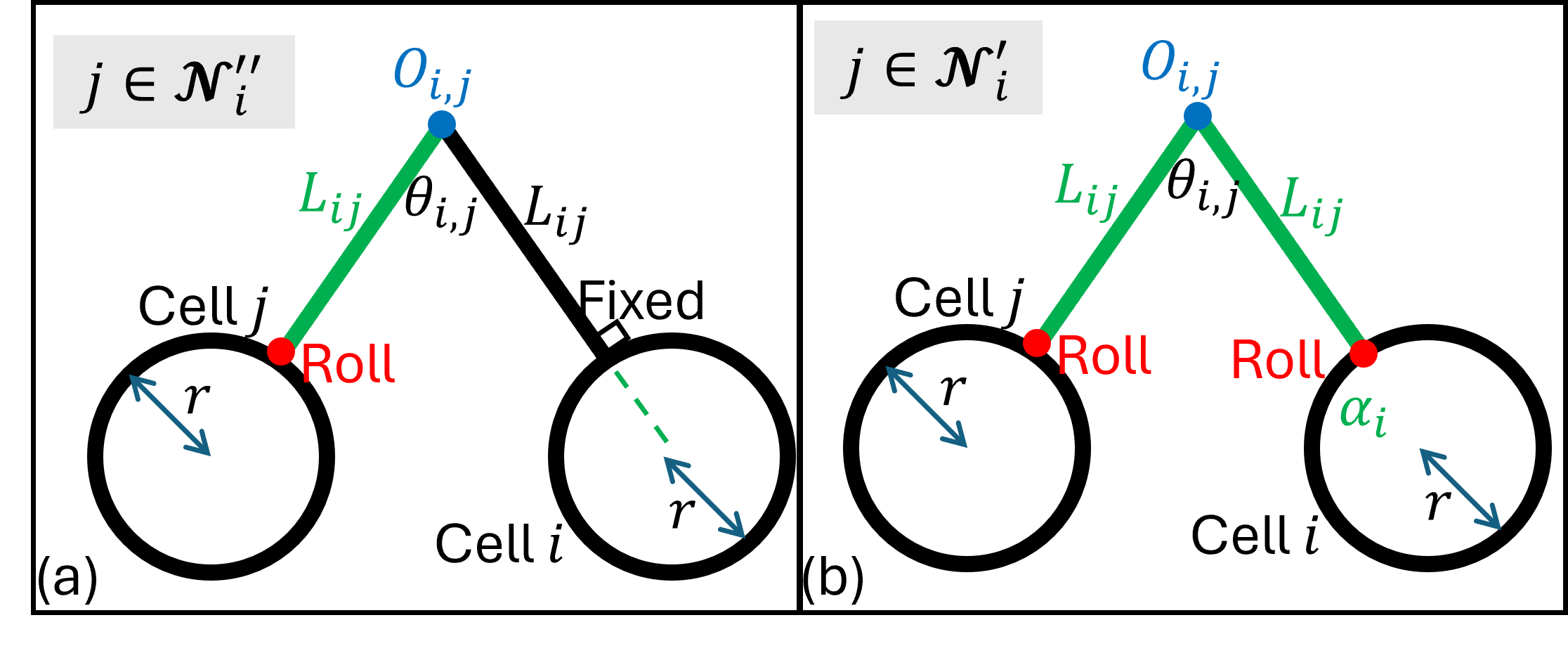}
\caption{The mechanisms used for connecting $i\in \mathcal{V}$ to $j\in \mathcal{N}_i$.}
\label{BarConnection} 
\end{figure}

\subsection{ATUGV Reference Configuration}
The reference configuration of the ATUGV specifies defines as a configuration at which all elbow angle $\theta_{i,j}=0$, for every $i\in \mathcal{V}$ and every $j\in \mathcal{N}_i$. The reference configuration of the ATUGV is defined as the configuration in which all elbow angles satisfy $\theta_{i,j}=0$  for every $i\in\mathcal{V}\setminus \mathcal{V}_0$ and every $j\in \mathcal{N}_i$.

Without loss of generality, the reference configuration of the agent team is distribued over an equilateral triangle with vertices occupied by three boundary cells $\mathcal{V}_0$. By using $\mathbf{a}_i$ to denote the reference position of cell $i\in \mathcal{V}$, reference position of the interior cells are determined by
\begin{equation}
    \mathbf{a}_i={1\over 3}\sum_{j\in \mathcal{N}_i}\mathbf{a}_j.
\end{equation}
Figure \ref{PUConnections} shows the reference configuration of a seven-cell ATUGV when all elbow angles are zero.
 \subsection{Cell Interconnection Mechanism}
 Every cell $i\in \mathcal{V}\setminus \mathcal{V}_0$ needs to be connected to a cell $j\in \mathcal{N}_i$ in such a way that 
 every ATUGV cell can freely move in the deformation plane. In other words, the motion of any cell in the deformation plane must not be  locked because of an indeterministic connection between two cells. Additionally, the powered and unpowered  cells should be able to rotate in the deformation plane  when the ATUGV deforms. To achieve this,  we desire that the rotation of every powered or unpowered cell is deterministic, and it can be deterministically specified based on cell positions. 

To satisfy these key requirements, we use the mechanisms shown in {\color{black}Fig. \ref{BarConnection} }  to connect $\mathcal{N}_i$ cells to $i\in \mathcal{V}\setminus \mathcal{V}_0$. More specifically, one cell of $\mathcal{N}_i$ is defined by $\mathcal{N}_i^{''}$ and connected to $i\in \mathcal{V}\setminus \mathcal{V}_0$ through the mechanism shown in Fig.  \ref{BarConnection}(a). The remaining two cells of $\mathcal{N}_i$ are defined by $\mathcal{N}_i^{'}$ and connected to $i\in \mathcal{V}\setminus \mathcal{V}_0$ through the mechanism shown in Fig.  \ref{BarConnection}(b).

Note that the two arms used for connecting $j\in \mathcal{N}_i$ to $i\in \mathcal{V}\setminus \mathcal{V}_0$ both have the same length $L_{i,j}$. Let $d_{i,j}(t)$ denote the distance between cells $i\in \mathcal{V}\setminus \mathcal{V}_0$ and $j\in \mathcal{N}_i$, then,
\begin{equation}
    \theta_{i,j}(t)=2\sin^{-1}\left({d_{i,j}(t)\over 2\left(L_{i,j}+r\right)}\right)
\end{equation}
is the angle between the two arms connecting cells (Fig. \ref{BarConnection}). 






Joint $O_{i,j}$ is free to rotate if the bar mechanisms {\color{black}shown in Fig. \ref{BarConnection} } is used to connect the powered cell $i\in \mathcal{V}$ to the powered cell $j\in \mathcal{N}_i$.


For better clarification, consider the ATUGV shown in Fig. \ref{DGUConfiguration} where cell $4\in \mathcal{V}_1$ is connected to the three cells defined by $\mathcal{N}_4=\left\{1,2,3\right\}\subset \mathcal{V}_0$. As illustrated in Fig.  \ref{DGUConfiguration}, the bar mechanisms connecting cells $1$ and $2$ to cell $4$ are controlled by the stepper motors at joints $O_{4,1}$ and $O_{4,2}$. Therefore, $\mathcal{N}_4^{'}=\left\{1,2\right\}$.

\subsection{Connection between Mobile Robot and Powwered Cell}
A mobile robot contained by a powered cell is connected to the inner rail through a bar that can slide along the inner rail of the enclosing powered cells. This bar can freely slide along the inner rail of the enclosing powered cell. Fig.  \ref{innercolnnection} illustrated how a mobile robot is connected to the enclosing powered cell.

\begin{figure}[ht]
\centering
\includegraphics[width=0.48 \textwidth]{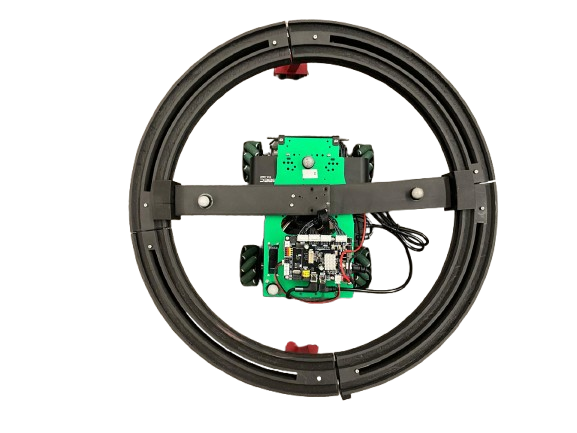}
\caption{Connection between a mobile mobot and the powered cell through a bar that can freely slide along the inner rail of the enclosing cell.}
\label{innercolnnection} 
\end{figure}


\section{ATUGV Motion Planning and Control}\label{atrans} 
The desired coordination of the ATUGV is defined by affine transformation
\begin{equation}\label{affinetransformation}
    \mathbf{p}_i=\mathbf{Q}(t)\mathbf{a}_i+\mathbf{d}(t),\qquad \forall i\in \mathcal{V},
\end{equation}
where $\mathbf{a}_i=\begin{bmatrix}X_i&Y_i\end{bmatrix}^T$ and $\mathbf{p}_i(t)=\begin{bmatrix}u_i(t)&v_i(t)\end{bmatrix}^T$ denote reference and desired positions of $i\in \mathcal{V}$, respectively, $\mathbf{Q}\in \mathbb{R}^{2\times 2}$ is the Jacobian matrix and $\mathbf{d}=\begin{bmatrix}
    d_1(t)&d_2(t)
\end{bmatrix}^T\in \mathbb{R}^{2\times 1}$ is the rigid-body displacement vector. 
Using the polar decomposition, $\mathbf{Q}(t)$ can be expressed as 
\begin{equation}
    \mathbf{Q}=\mathbf{R}\left(\sigma_r\right)\mathbf{U}\left(\sigma_d,\lambda_1,\lambda_2\right),
\end{equation}
where $\sigma_r(t)$ and $\sigma_d(t)$ and rigid-body rotation and shear deformation angles,  $\lambda_1(t)\in (0,1]$ is the first principal strain and $\lambda_2(t)\in (0,1]$ is the second principal strain, all at time $t$. Also,
\begin{equation}
    \mathbf{R}=\begin{bmatrix}
        \cos\sigma_r&-\sin\sigma_r\\
        \sin\sigma_r&\cos \sigma_r
    \end{bmatrix}
\end{equation}
is the rigid-body rotation matrix, and 
\begin{equation}
    \mathbf{U}=\begin{bmatrix}
        \cos\sigma_d&-\sin\sigma_d\\
        \sin\sigma_d&\cos \sigma_d
    \end{bmatrix}
    \begin{bmatrix}
        \lambda_1&0\\
        0&\lambda_2
    \end{bmatrix}
    \begin{bmatrix}
        \cos\sigma_d&\sin\sigma_d\\
        -\sin\sigma_d&\cos \sigma_d
    \end{bmatrix}
\end{equation}
is the positive definite strain matrix. Affine transformation of the ATUGV can be planned by specifying generalized coordinates $\lambda_1(t)$, $\lambda_2(t)$, $\sigma_r(t)$, $\sigma_d(t)$, $d_1(t)$ and $d_2(t)$, that are considered as generalized coordinates of an affine transformation.

The paper's main objective is that all cells are treated as particles of a deformable body where the actual position of every cell $i\in \mathcal{V}$, denoted by $\mathbf{r}_i(t)$,  stably tracks the desired position $\mathbf{p}_i(t)$, defined by Eq. \eqref{affinetransformation}. 
\subsection{Affine Deformation Constraints}
There exists a lower bound  on the  principal strains $\lambda_1$ and $\lambda_2$ of the affine transformation that are determined based on the minimum speration distance in the reference configuration of ATUGV and the cell radius $r$. More specifically, $\lambda_1(t)\in \left[\lambda_{min},1\right]$ and $\lambda_2(t)\in \left[0,1\right]$ at any time $t$ wehere $\lambda_{min}>0$ is obtained by Theorem \ref{thm}.
\begin{theorem}\label{thm}
    Let $d_{min}$ denote the minimum sepration distance between every two cells in the reference configuration, and $r$ denote the same radius of every ATUGV cell. Then,
    \begin{equation}\label{lmin}
        \lambda_{min}={2r\over d_{min}}
    \end{equation}
    assigns a lower bound for the principal strains of the affine transformations.
\end{theorem}
\begin{proof}
    When cells' desired trajectories are defined by an affine transformation, the desired positions of every two cells $i$ and $j$ holds the following condition:
\begin{equation}
\|\mathbf{p}_i-\mathbf{p}_j\|^2=\left(\mathbf{a}_i-\mathbf{a}_j\right)^T\mathbf{Q}^T\mathbf{Q}\ \left(\mathbf{a}_i-\mathbf{a}_j\right),\qquad i,j\in \mathcal{V},
\end{equation}
where 
\begin{subequations}
\begin{equation}
\min\limits_{i\neq j} \|\mathbf{p}_i-\mathbf{p}_j\|^2=4r^2
\end{equation}
\begin{equation}
\left(\mathbf{a}_i-\mathbf{a}_j\right)^T\mathbf{Q}^T\mathbf{Q}\ \left(\mathbf{a}_i-\mathbf{a}_j\right)\geq \lambda_{min}^2\|\mathbf{a}_i-\mathbf{a}_j\|^2\geq \lambda_{min}2^2 d_{min}^2,
\end{equation}
and 
\begin{equation}
\lambda_{min}=\min\left\{\lambda_1,\lambda_2\right\}.
\end{equation}
Therefore, no two cells collide, if
\[
4r^2\geq \lambda_{min}^2d_{min}^2
\]
\end{subequations}
which in turn implies that Eq. \eqref{lmin} assigns a lower bound for the principal strains $\lambda_1(t)$ and $\lambda_2(t)$ at any time $t$.

\end{proof}




\subsection{Affine Transformation Tracking by Powered Cells}

Every powered cell knows desired position $\mathbf{p}_i(t)$ at any time $t$.
Also, every robot $i\in \mathcal{V}\setminus \mathcal{V}_{N_L}$ knows its own position with respect to the global coordinate system. More precisely, UGV coordinates in an indoor robotic facility equipped with a high-resolution motion capture system where the actual position of every cell $i\in \mathcal{V}\setminus  \mathcal{V}_{N_L}$ is precisely measured and communicated to the robot at any time. Therefore, every powered cell $i\in \mathcal{V}$ can accurately track the desired trajectory $\mathbf{p}_i(t)$ by commanding the velocity as follows:
\begin{equation}
    \mathbf{v}_i=\alpha_i\left(\mathbf{p}_i-\mathbf{r}_i\right)
\end{equation}
where $\alpha_i$ is the tuning parameter that are assigned by experiments.
\subsection{Affine Transformation Tracking by Unpowered Cells}
The position of an unpowered cell $i\in \mathcal{V}$ is controlled by controlling the elbow angle $\theta_{i,j}$, for every $i\in \mathcal{V}\setminus \mathcal{V}_0$ and every $j\in \mathcal{N}_i$.
Defining $\mathcal{N}_i=\left\{j_1,j_2,j_3\right\}$ and $\mathcal{N}_i^{'}=\left\{j_1,j_2\right\}$, two stepper motors are installed in the joints $O_{i,j_1}$ and $O_{i,j_2}$ to control the angles of the elbows denoted by $\theta_{i,j_1}(t)$ and $\theta_{i,j_2}(t)$. The desired elbow angles are denoted by   $\theta_{i,j_1}^d(t)$ and  $\theta_{i,j_1}^d(t)$, commanded to the stepper motors, are determined based on the desired positions of agents $\mathbf{p}_i(t)$, $\mathbf{p}_{j_1}(t)$, and $\mathbf{p}_{j_2}(t)$ as follows:
\begin{subequations}
    \begin{equation}
    \theta_{i,j_1}^d(t)=2\sin^{-1}\left({\|\mathbf{p}_i(t)-\mathbf{p}_{j_1}(t)\|\over 2\left(L_{i,j}+r\right)}\right)
\end{equation}
\begin{equation}
    \theta_{i,j_2}^d(t)=2\sin^{-1}\left({\|\mathbf{p}_i(t)-\mathbf{p}_{j_2}(t)\|\over 2\left(L_{i,j}+r\right)}\right)
\end{equation}
\end{subequations}

\begin{figure}[ht]
\centering
\includegraphics[width=0.48 \textwidth]{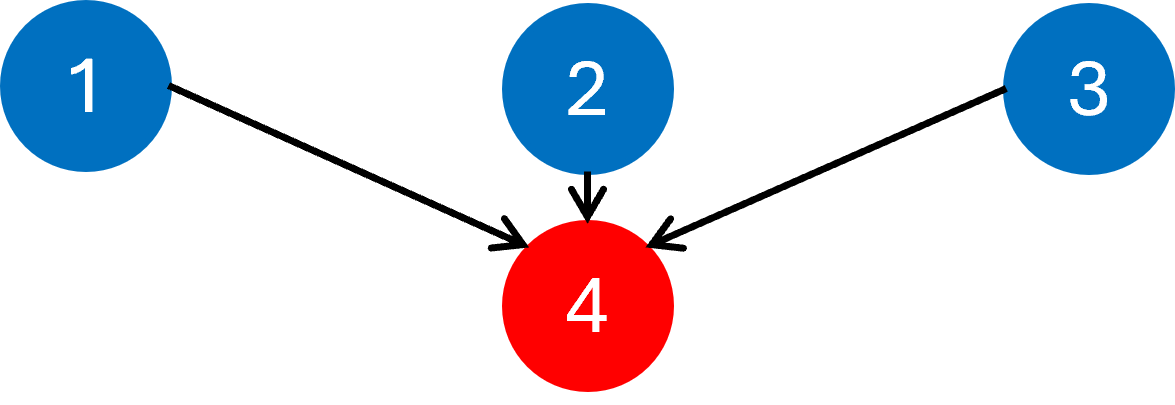}
\caption{An example ATUGV with three powered cells and one unpowered cell.}
\label{DGUConfigurationnn} 
\end{figure}
\begin{table}[]
    \centering
    \caption{Affine transformation design parameters for for the hardware-bases and simulation-based cases studies.}
    \begin{tabular}{|c|c|c|c|c|c|c|c|}
        \hline
         &$\lambda_1^f$&$\lambda_2^f$&$d_1^f$&$d_2^f$&$\sigma_r^f$&$\sigma_d^f$& $t_f$\\        
        \hline
        Experiment &0.9&0.8&1m&1m&0.2rad&0.15 rad &20s\\
         \hline
        Simulation &0.9&0.8&1m&1m&0.707rad&0.3 rad &10s\\
        \hline
    \end{tabular}    
    \label{tab:my_label}
\end{table}
\section{Results}\label{results}
The functionality of the proposed ATUGV system is evaluated through hardware experimentation and simulation with the results presented in Sections \ref{Experimental Results} and \ref{Simulation Results}, respectively. The affine transformation of the ATUGV is planned by defining generalized coordinates $\lambda_1$, $\lambda_2$, $d_1$, $d_2$, $\sigma_r$, and $\sigma_d$ over a finite horizon of time as follows:
\begin{subequations}
    \begin{equation}
        \lambda_1(t)=\left(1-\beta(t)\right)\lambda_1^0+\beta(t)\lambda_1^f,\qquad t\in \left[t_0,t_f\right],
    \end{equation}
     \begin{equation}
        \lambda_2(t)=\left(1-\beta(t)\right)\lambda_2^0+\beta(t)\lambda_2^f,\qquad t\in \left[t_0,t_f\right],
    \end{equation}
    \begin{equation}
        d_1(t)=\left(1-\beta(t)\right)d_1^0+\beta(t)d_1^f,\qquad t\in \left[t_0,t_f\right],
    \end{equation}
    \begin{equation}
        d_2(t)=\left(1-\beta(t)\right)d_2^0+\beta(t)d_2^f,\qquad t\in \left[t_0,t_f\right],
    \end{equation}
    \begin{equation}
        \sigma_r(t)=\left(1-\beta(t)\right)\sigma_r^0+\beta(t)\sigma_r^f,\qquad t\in \left[t_0,t_f\right],
    \end{equation}
    \begin{equation}
        \sigma_d(t)=\left(1-\beta(t)\right)\sigma_d^0+\beta(t)\sigma_d^f,\qquad t\in \left[t_0,t_f\right],
    \end{equation}
\end{subequations}
where we choose $t_0=0$, $\lambda_1^0=\lambda_2^0=1$, $d_1^0=d_2^0=0$, and $\sigma_r^0=\sigma_d^0=0$ for both case studies. The final time $t_f$, $\lambda_1^f$, $\lambda_2^f$, $d_1^f$, $d_2^f$, $\sigma_r^f$, and $\sigma_d^f$ are reported in Table  \ref{tab:my_label} for the hardware experimentation and simulation. We note that $\beta:\left[t_0,t_f\right]\rightarrow \left[0,1\right]$ is an increasing function satisfying the following conditions $\beta(t_0)=0$ and $\beta(t_f)=1$.


\subsection{Experimental Results}\label{Experimental Results}

The experiments are carried out in the indoor testing facility measuring $5$m x $5$m x $2$m in the S.M.A.R.T Lab at the University of Arizona equipped with Vicon Motion Capture system for indoor positioning.
\begin{figure}[ht]
\centering
\includegraphics[width=0.48 \textwidth]{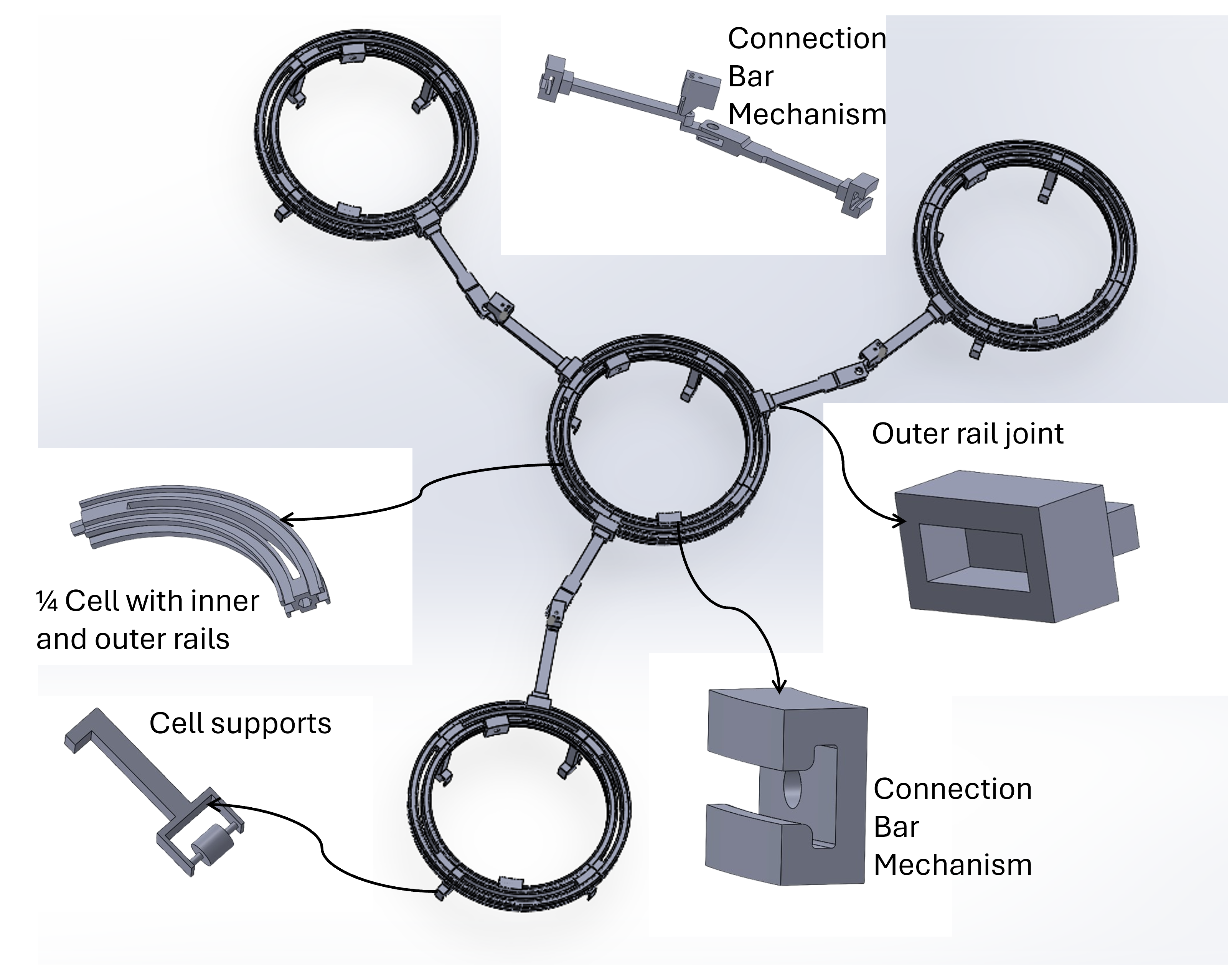}
\caption{An example ATUGV with three powered cells and one unpowered cell.}
\label{DeformableStructure} 
\end{figure}
\subsubsection{Hardware Setup}
For the hardware experimentation, we conduct tests on the operation of a four-cell ATUGV, as illustrated in Fig. \ref{DGUConfiguration}, in conjunction with the deformable structure designed in SolidWorks, depicted in Fig. \ref{DeformableStructure}. The deformable structure, which was fabricated using the available 3D printers in the S.M.A.R.T lab, comprises five primary components, as detailed in Fig. \ref{DeformableStructure}.

\begin{figure}[ht]
\centering
\includegraphics[width=0.48 \textwidth]{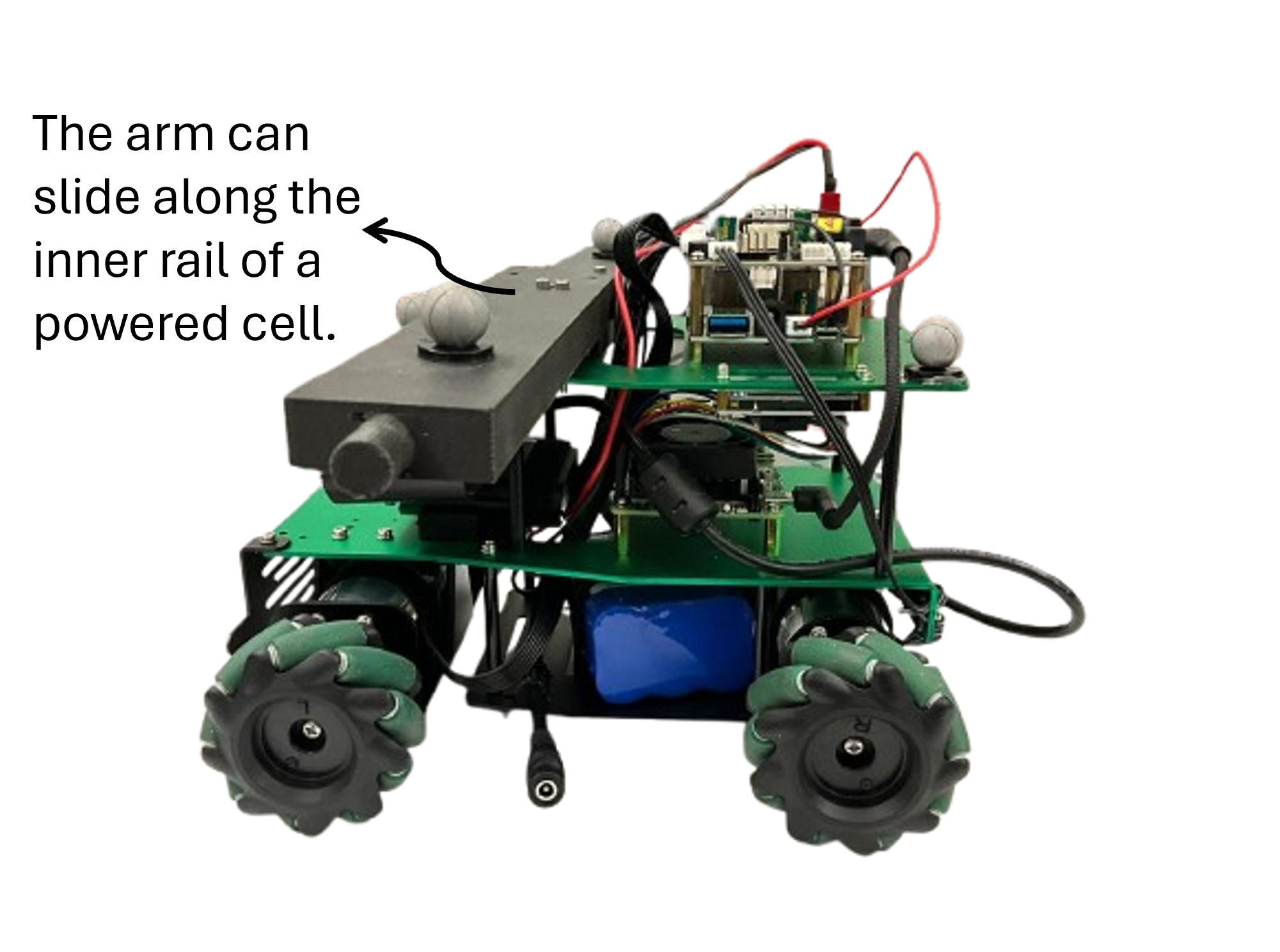}
\caption{A ROSMASTER X3 mobile robot is used to control the motion of the powered cells. A mobile robot is connected to the enclosing cell through a bar that can slide along the inner rail of the enclosing cell.}
\label{main3} 
\end{figure}

For our experiment,  every powered cell encloses 
 ROSMASTER X3 mobile robots to generate the required force for tracking the desired affine transformation by the entire powered and unpowered cells (Our ATUGV system, shown in Fig. \ref{DGUConfiguration},   consists of three powered cells (Cell 1, Cell 2, and Cell 3), each equipped with a ROSMASTER X3, while Cell 4 remains unpowered and moves passively based on the forces exerted by the other cells). Note that each mobile robot is connected to the enclosing powered cell  through the arm that is attached on top of the robot (see Figs. \ref{main3} and \ref{innercolnnection}).

 By using Mecanum wheels, these robots can move in any direction, allowing smooth and efficient force application. Their ability to adjust motor speeds and torque helps maintain balance and stability while executing different movement patterns. Through active motion control, force distribution, and multi-robot coordination, the robots work together to influence system movement and enable adaptable navigation.

The force generation mechanism of the ROSMASTER X3 robots is based on motion synchronization and adaptation to external forces. The powered cells apply controlled forces to drive linear, rotational, and sideways motion, while the unpowered cell moves in response to the combined forces. The Mecanum wheel system allows fine adjustments, ensuring that force vectors align with the system's movement needs. This adaptability is further improved by real-time sensor feedback, which enables ROS2-based corrections for trajectory control and obstacle avoidance. By varying their speed and movement patterns, the robots enable smooth path following, coordinated system rotation, and efficient navigation, making them suitable for advanced robotic applications.
\begin{figure}[h]
\centering
\subfigure[]{\includegraphics[width=0.49\linewidth]{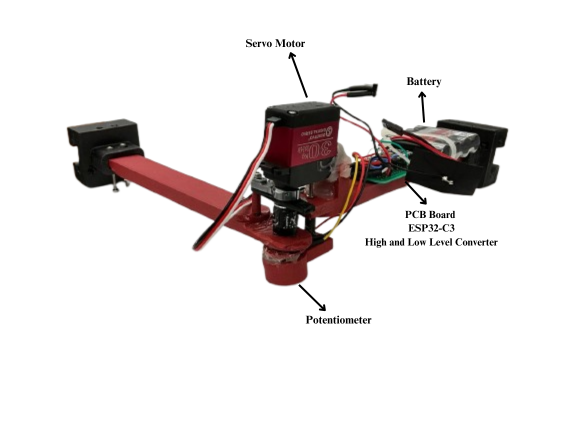}}
\subfigure[]{\includegraphics[width=0.49\linewidth]{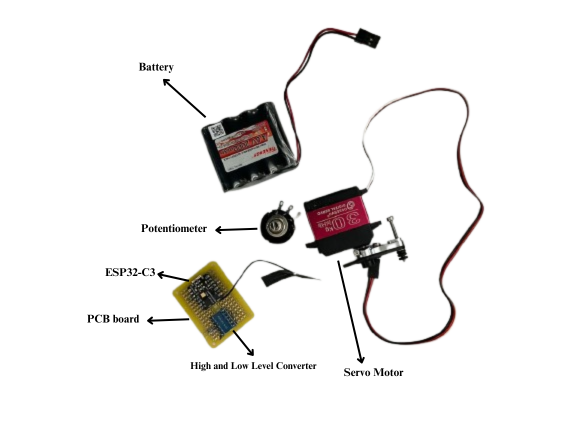}}
\vspace{-0.4cm}
\caption{(a) The servo motor used for controlling the elbow angles. (b) The components of the elbow angle controller.
}
\label{armsservomotor}
\end{figure}
The integration of ROSMASTER X3 with ROS2 ensures scalability and efficient force application, allowing the system to adapt to different conditions. Their precise motor control and ability to move in any direction make them ideal for applications such as indoor navigation, swarm robotics, and industrial automation. The ability to coordinate movement among multiple robots allows real-world deployment in structured environments, ensuring that each robot contributes effectively to the overall movement of the system. By adjusting their forces in real time and communicating efficiently, the ROSMASTER X3 robots form an effective multirobot force generation system capable of executing complex motion strategies with precision.

The controllable servo motor arm is an essential subsystem within the system, designed for precise motion control and position measurement (see Fig. \ref{armsservomotor}). It is powered by an ESP32-C3 microcontroller, which generates PWM signals to control a 30kg digital servo motor, enabling high-torque actuation with controlled movement. A potentiometer is integrated to provide real-time positional feedback, facilitating a closed-loop control mechanism that enhances accuracy and stability. A high and low-level converter ensures proper voltage regulation between the microcontroller and other components, protecting the electronics and optimizing signal transmission. The entire system is powered by a battery pack, ensuring portability and adaptability for various experimental setups. The structural framework is built using 3D-printed and metal components, offering a balance between lightweight flexibility and mechanical durability.

\subsubsection{Software Setup}

The system's software architecture has been meticulously designed to facilitate seamless interaction among the planning module, control algorithms, and hardware components (see Fig. \ref{ATUGV-ROS}). The motion capture system is the principal external sensor for localization, capturing pose data, which is subsequently processed at the ground station. The mediator node, which functions as a central dispatcher, utilizes this processed data, ensuring the planner's load is minimized. The planner generates the desired trajectory, which the mediator thereafter processes. The mediator operates the velocity and arm controller to direct the agents and the arm to the desired location. The velocity controller converts the desired trajectory into velocity commands for the agents, while the arm controller guarantees precise positioning of the arms through motor commands.

Figure \ref{experiment} shows the actual and desired paths of cells $1$ through $4$. The video of our experiment can be found \href{https://www.youtube.com/watch?v=KO81Bpweirc}{here}.

\begin{figure}[ht]
\centering
\includegraphics[width=0.48 \textwidth]{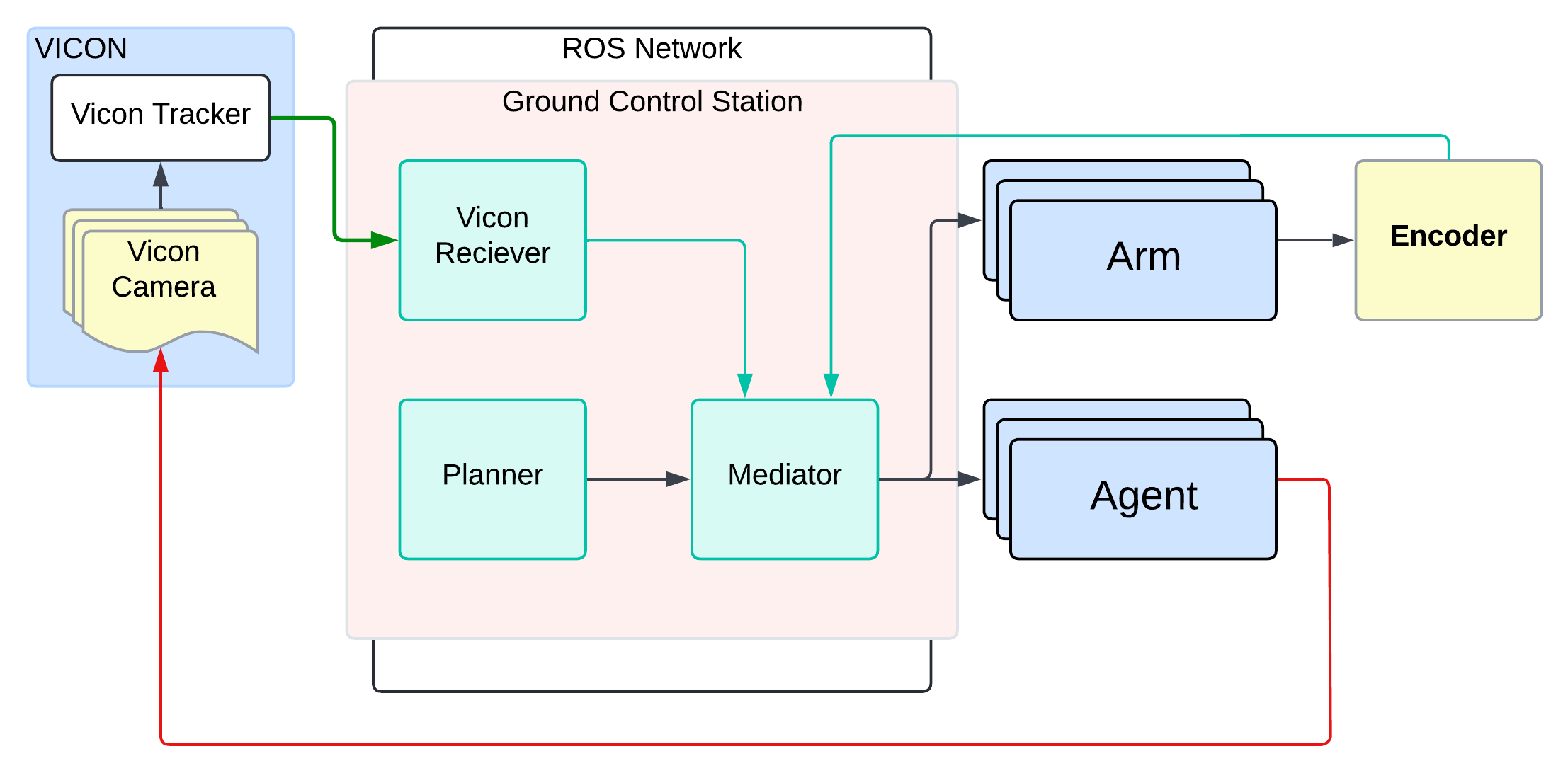}
\caption{Overview of the software stack.}
\label{ATUGV-ROS} 
\end{figure}

\begin{figure}[ht]
\centering
\includegraphics[width=0.48 \textwidth]{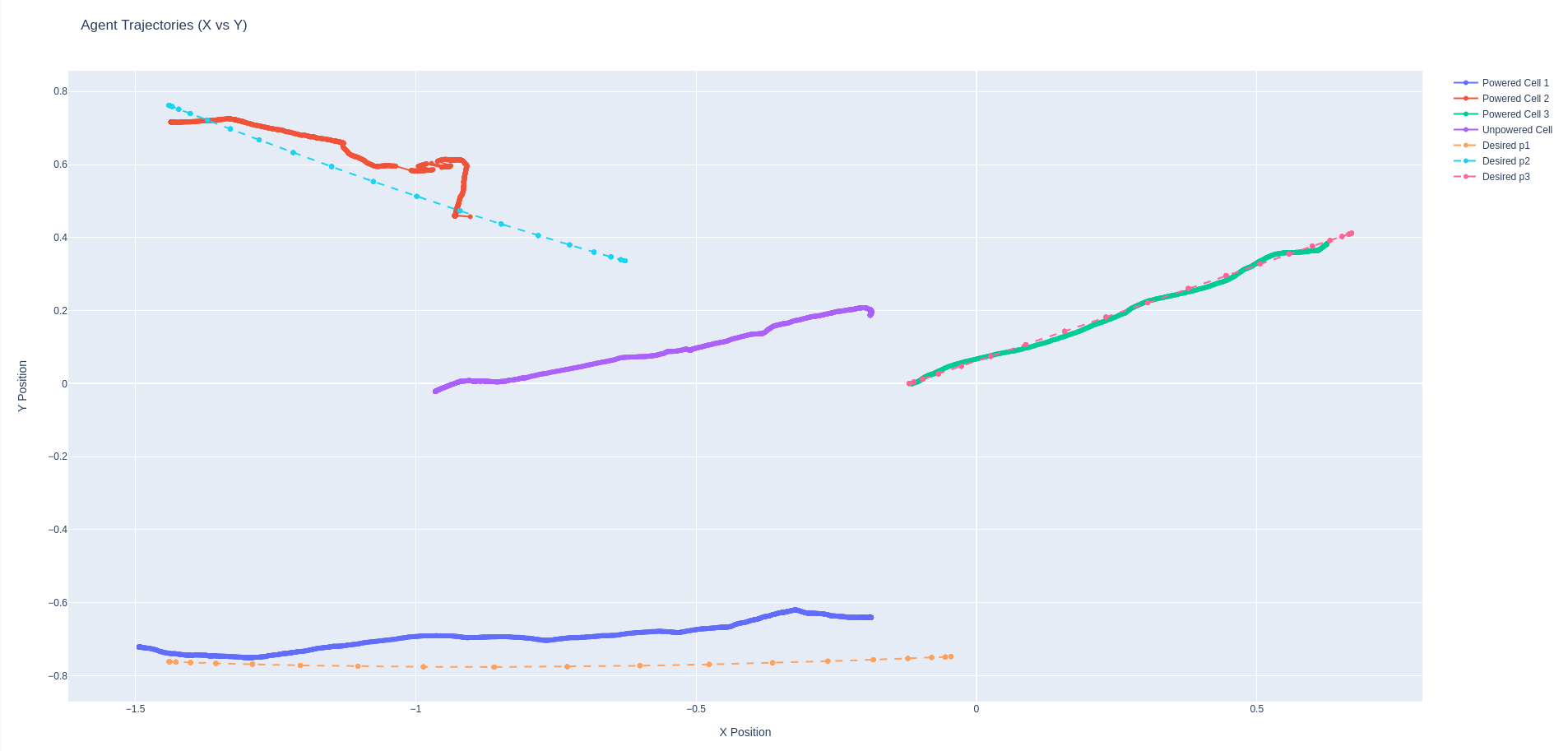}
\caption{Actual and desired path of cells $1$ through $4$ obtained through ATUGV affine transformation experiment.}
\label{experiment} 
\end{figure}


\subsection{Simulation Results}\label{Simulation Results}
For simulation, we evaluated the affine transformation of a seven-cell ATUGV with the reference configuration and interconnection topology shown in Figs. \ref{PUConnections} (a) and (b), respectively. where cells $1$, $2$, $3$, and $4$ are powered and the remaining cells $5$, $6$, and $7$ are unpowered.
Given the final time $t_f=10$, $\lambda_1^f$, $\lambda_2^f$, $d_1^f$, $d_2^f$, $\sigma_r^f$, and $\sigma_d^f$ listed in Table \ref{tab:my_label}, the final configuration of the ATUGV is shown in Fig. \ref{DGUConfigurationSeven}. For simulation, we assume that every powered cell is operated by a double-integrator UGV. The $X$ and $Y$ components of all cell positions are plotted versus time in Fig.  \ref{XY}.


\begin{figure}[ht]
\centering
\includegraphics[width=0.48 \textwidth]{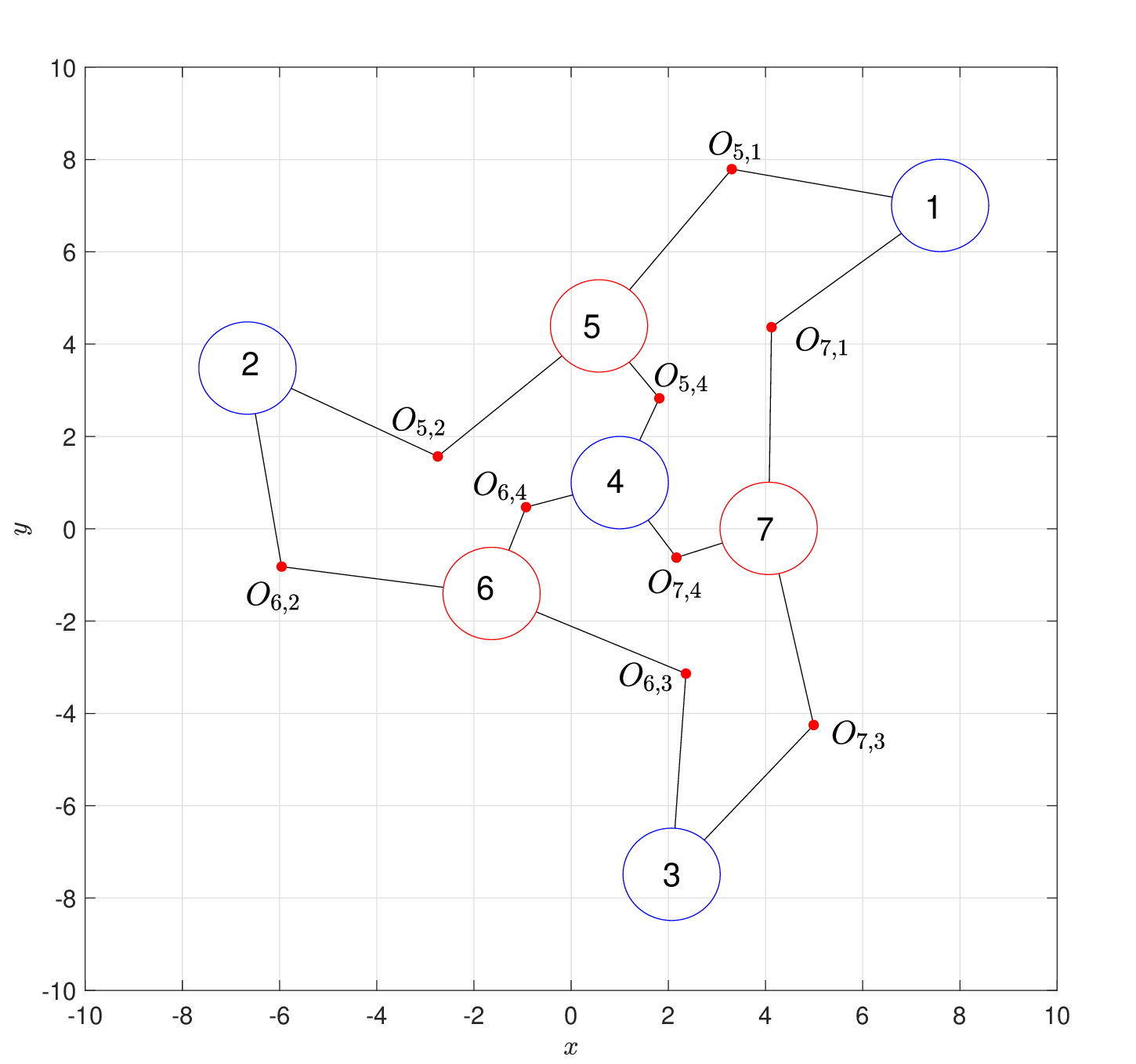}
\caption{Final configuration of the seven-cell aTUGV.}
\label{DGUConfigurationSeven} 
\end{figure}

\begin{figure}[h]
\centering
\subfigure[]{\includegraphics[width=0.98\linewidth]{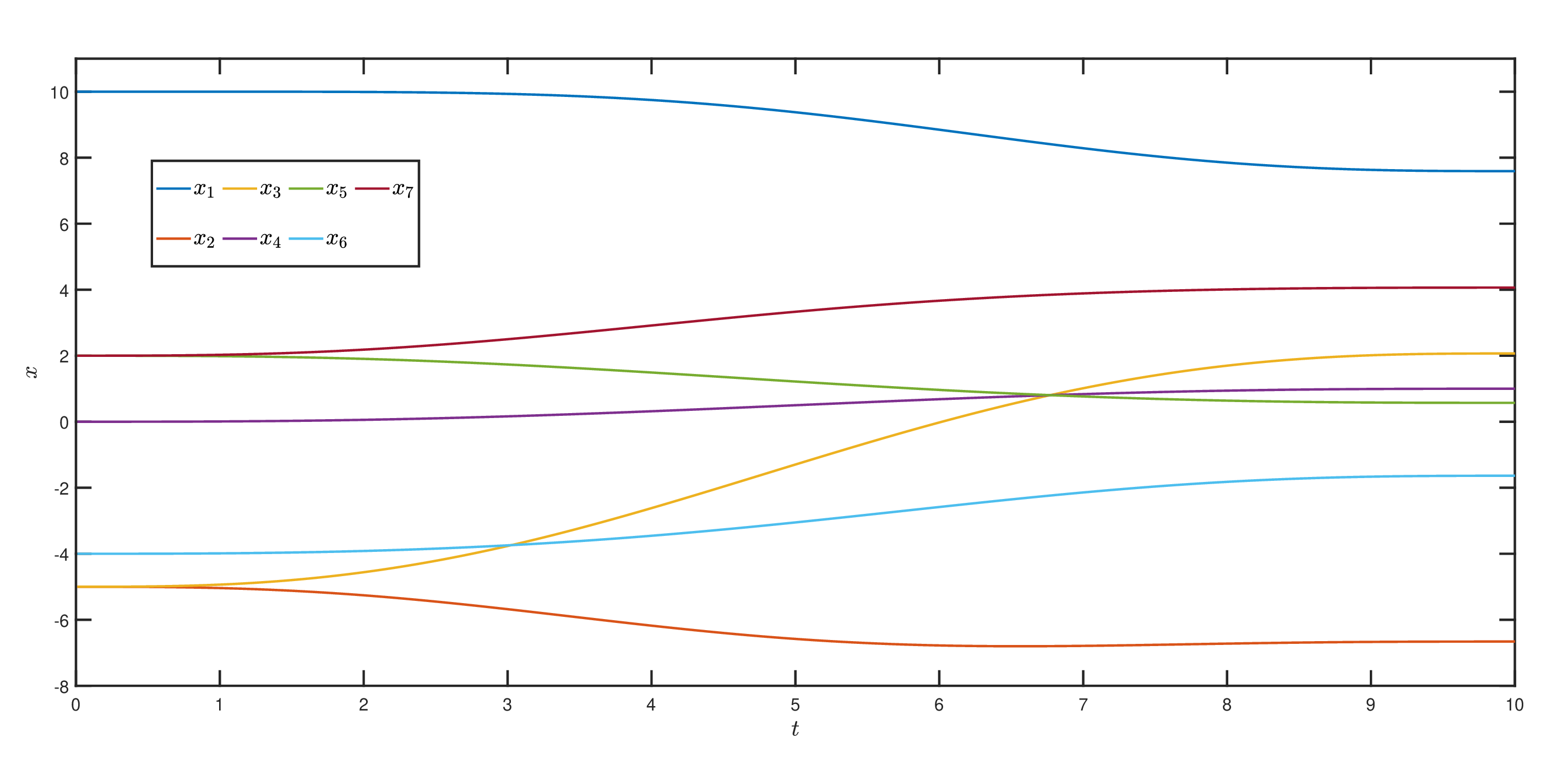}}
\subfigure[]{\includegraphics[width=0.98\linewidth]{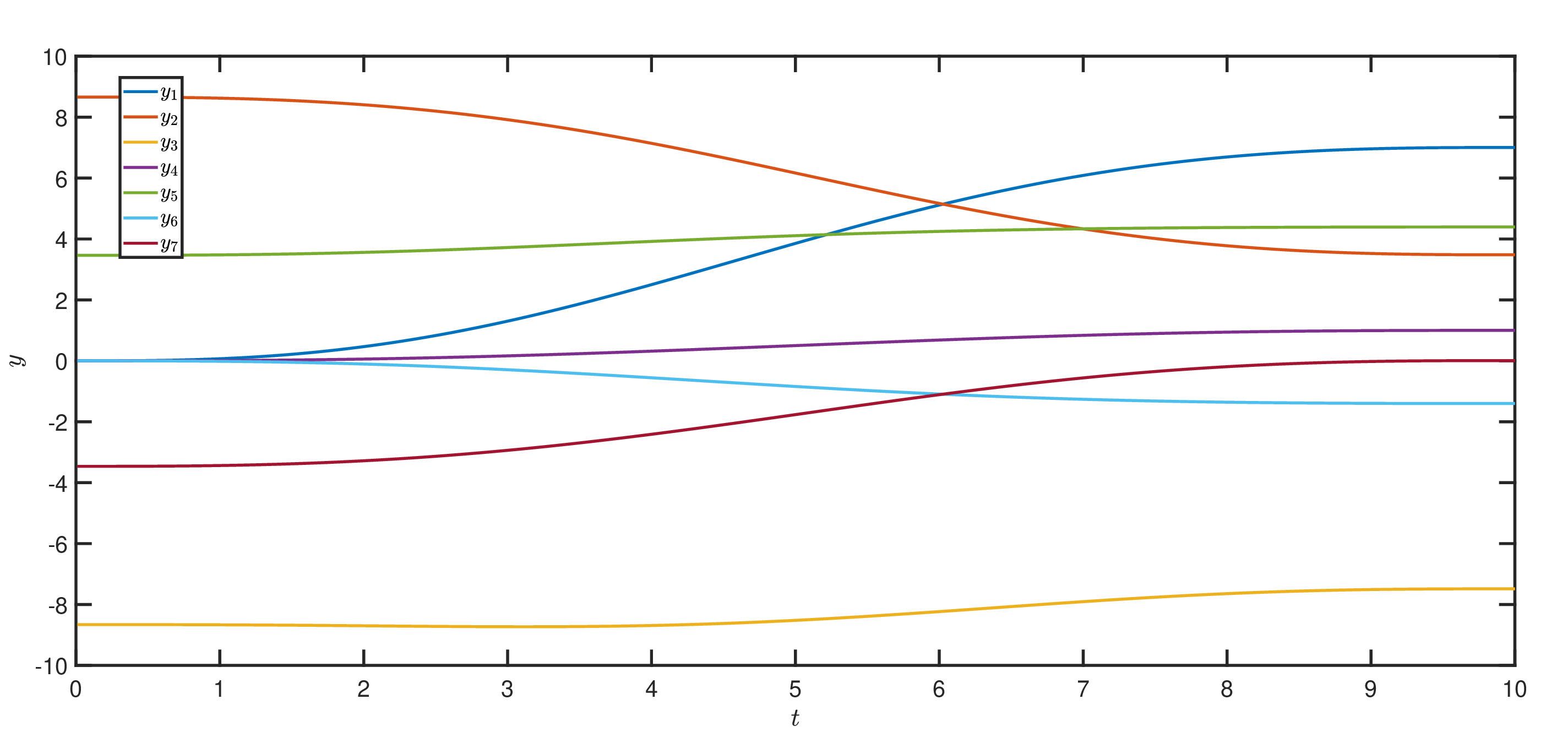}}
\vspace{-0.4cm}
\caption{$x$ and $y$ components of cell positions versus time.
}
\label{XY}
\end{figure}

\section{Conclusion and Future Work}\label{Conclusion}
The paper presents proof of concept for a new morphable mobile robot with capability of resilient and robust operation in a constrained environment. While ATUGV can have an aggressive planar deformation, it will offer structural rigidity in every plane normal to the deformation plane of the vehicle. As a result, the proposed ATUGV should enhance the robustness of operations of existing UGVs against disturbance forces during inclement environmental conditions. Furthermore, the ATUGV will benefit from redundant power due to multiple integrated quadcopters, to advance the resilience of the DCU operation against failure and improve the safety recovery capabilities of UAVs. By carrying mutiple paylaods and offering the capability of reconfiguration, ATUGV can improve the scalability of performance of mobile robots for logistics applications. 

In this paper, we assumed that all mobile robots can localize themselves and know their own desired trajectories. For future work, we relax this constrained  by enabling ATUGV to acquire a desired affine transformation in a decentralized fashion through cooperative localization.
\section{Acknowledgement}
The authors greatfully thank Tech Launch Arizona for supporting this work.

\bibliographystyle{IEEEtran}
\bibliography{reference}


%

\appendices




\ifCLASSOPTIONcaptionsoff
  \newpage
\fi

\end{document}